\documentclass[journal]{IEEEtran}
\usepackage[english]{babel}
\usepackage{color}
\usepackage{graphicx}
\usepackage[export]{adjustbox}
\usepackage[caption=false]{subfig}	
\usepackage{ifpdf}
\ifpdf
  \usepackage{epstopdf}
\fi

\usepackage{amsmath,mathtools,amsthm,bbm}
\usepackage{amssymb,amsfonts,mathrsfs}

\usepackage{array,calc}
\usepackage{enumerate,enumitem}
\usepackage{cite,cleveref}
\crefname{figure}{fig.}{figs.}

\newcommand{\wid}{\mathscr{W}}
\DeclareMathOperator{\sign}{sign}

\crefname{app}{Appendix}{Appendices}
\crefname{cor}{Corollary}{Corollaries}
\crefname{prop}{Proposition}{Propositions}
\crefname{lemma}{Lemma}{Lemmas}
\crefname{defn}{Definition}{Definitions}
\crefname{conj}{Conjecture}{Conjectures}
\crefname{exam}{Example}{Examples}
\crefname{supp}{Supplemental Section}{Supplemental Sections}

\newtheorem{theorem}{Theorem}
\newtheorem{cor}{Corollary}
\newtheorem{lemma}{Lemma}
\newtheorem{prop}{Proposition}

\newtheorem{defn}{Definition}

\newcommand{\bs}{\boldsymbol}
\newcommand{\bb}{\mathbb}
\newcommand{\mcal}{\mathcal}

\newcommand{\eye}{\bs{I}}
\newcommand{\zero}{\bs{0}}

\newcommand{\lb}{\left(}
\newcommand{\rb}{\right)}
\newcommand{\ls}{\left[}
\newcommand{\rs}{\right]}
\newcommand{\lc}{\left\{}
\newcommand{\rc}{\right\}}

\newcommand{\lv}{\left\vert}
\newcommand{\rv}{\right\vert}
\newcommand{\lV}{\left\Vert}
\newcommand{\rV}{\right\Vert}

\newcommand{\LRV}[1]{{\left\vert\kern-0.25ex\left\vert\kern-0.25ex\left\vert #1 \right\vert\kern-0.25ex\right\vert\kern-0.25ex\right\vert}}

\newcommand{\expect}[2]{\bb{E}_{#1}\lc#2\rc}

\newcommand{\tran}{^{\mathsf{T}}}

\newcommand{\Gauss}{\mathcal{N}}

\newcommand{\matA}{\bs{A}}

\newcommand{\bbP}{\bb{P}}

\newcommand{\bbR}{\bb{R}}

\newcommand{\calB}{\mcal{B}}

\newcommand{\calN}{\mcal{N}}

\newcommand{\calS}{\mcal{S}}
\newcommand{\calT}{\mcal{T}}

\newcommand{\vecg}{\bs{g}}

\newcommand{\vect}{\bs{t}}
\newcommand{\vecu}{\bs{u}}
\newcommand{\vecv}{\bs{v}}

\newcommand{\vecx}{\bs{x}}
\newcommand{\vecy}{\bs{y}}
\newcommand{\vecz}{\bs{z}}

\allowdisplaybreaks

\begin{document}

\title{One-bit Compressed Sensing using Generative Models}
\author{ 
	\author{Swatantra Kafle, Geethu Joseph, \textit{Senior Member, IEEE},  and Pramod K. Varshney, \textit{Life Fellow, IEEE}
    \thanks{The material in this paper was presented
in part at the IEEE International Conference on Acoustics, Speech, \& Signal Processing (ICASSP), Barcelona, Spain in May 2020~\cite{joseph2020one}.

S. Kafle is currently with ANDRO Computational Solutions, NY, USA. He was with the Electrical Engineering and Computer Science Department at Syracuse University,  NY, USA, during the course of this work. Email:~\texttt{swakafle@gmail.com}

G. Joseph is with the Signal Processing Systems group, Electrical Engineering, Mathematics, and Computer Science faculty, at the Delft University of Technology, The Netherlands. Email:~\texttt{g.joseph@tudelft.nl}.

P. K. Varshney is with the Electrical Engineering and Computer Science Department at Syracuse University,  NY, USA. Email:~\texttt{varshney@syr.edu}.
}

}
}

\maketitle
\begin{abstract}
This paper addresses the classical problem
of one-bit compressed sensing using a deep learning-based
reconstruction algorithm that leverages a trained generative
model to enhance the signal reconstruction performance. The
generator,  a pre-trained neural network, learns to map from
a low-dimensional latent space to a higher-dimensional set of sparse
vectors. This generator is then used to reconstruct sparse
vectors from their one-bit measurements by searching over its
range. The presented algorithm provides an excellent reconstruction performance because the generative model can learn additional structural information about the signal beyond sparsity. Furthermore, we provide theoretical guarantees on the reconstruction accuracy and sample complexity of the algorithm. Through numerical experiments using three publicly
available image datasets,  MNIST, Fashion-MNIST, and Omniglot, we demonstrate the superior performance of the algorithm compared to other existing algorithms and show that our algorithm can recover both the amplitude and the direction of the signal from one-bit measurements. 
\end{abstract}

\section{Introduction}
Over the past two decades, research in compressed sensing (CS)~\cite{rani2018systematic,qaisar2013compressive} has expanded rapidly, leading to advancements in signal reconstruction algorithms~\cite{OMP, BCS, CSBP, Cevher1, IHT} and inference tasks such as detection, estimation, and classification~\cite{Nowak, kafleDet, wimalajeewapartial, app}. The success of CS, coupled with the fundamental role of quantization in signal digitization, has fueled a growing interest in quantized CS\cite{li2018survey, quantized_dai, quantized_jacques}. Coarse quantization is particularly appealing as it results in significant reduction in bandwidth requirements and power consumption. One of the more popular quantization schemes is one-bit quantization, wherein the measurements are binarized by comparing signals/measurements to a fixed reference level. Using the zero reference level is the most used one-bit quantization scheme, which is also the focus of our paper. Here, the measurements are quantized based on their signs. The popularity of one-bit quantization stems from its simplicity, cost-effectiveness, and robustness to certain linear and nonlinear distortions, such as saturation~\cite{boufounos20081, laska2011trust}. However, one-bit measurements lose information about signal amplitude, and it is impossible to recover the amplitude during signal reconstruction. Therefore, one-bit CS finds applications in systems that require the recovery of the unknown signal up to a scaling factor. For example, in a frequency division duplex massive MIMO system, the direction of the channel state information at the transmitter is sufficient for the design of beam-forming vectors. In this case, using one-bit CS saves the uplink bandwidth resources required for the channel state information feedback~\cite{tang2017low}. Some other applications where one-bit CS is used are radar~\cite{ameri2019one}, source localization~\cite{shen2013one}, spectrum sensing~\cite{jian2011investigation}, and wireless sensor networks~\cite{cao2016implementation}. Motivated by these applications, in this paper, we focus on the one-bit CS problem, where the objective is to find an unknown sparse vector from its one-bit quantized noisy linear measurements.

\subsection{Related Literature}
One-bit CS was originally introduced in \cite{boufounos20081}, and several reconstruction algorithms have since then been proposed in the literature  \cite{kamilov, musa, kafle2016, binarystableEmbedd,plan2012robust,baraniukexponential, normEstimation,classification_binaryGAMP,kafle2019noisy, xu2014bayesian,kafle2021one}. One-bit CS is also known to outperform multi-bit quantized CS in some scenarios \cite{regimeCh_Laska}. Though one-bit CS has shown promising inference and signal reconstruction performance, it is also known to be quite sensitive to noise\cite{r1bcs, kafleLaplacian, yan2012robust}. Some recent works have dealt with the problem by mitigating noise \cite{kamilov, musa, plan2012robust}, using multiple measurement vectors \cite{kafle}, or by using side-information \cite{kafleLaplacian}. 
These algorithms fall into the category of ``traditional'' algorithms as they are model-driven, where the recovery performance depends on how well the model represents the actual sparse structure of the signal. Among model-driven algorithms, Bayesian algorithms often perform better than non-Bayesian counterparts due to their ability to incorporate the prior on the sparse signal structure through a probability distribution. Recently, another class of algorithms that uses the deep learning-based approach has gained traction in the literature \cite{bora2017compressed} for various estimation problems. Inspired by these advances, in this paper, we explore the possibility of using a deep learning-based approach for one-bit CS, i.e., we investigate the signal reconstruction performance of generative model-based one-bit CS. 

A well-studied deep learning technique for signal reconstruction is based on neural networks called generative models, such as generative adversarial networks ~\cite{goodfellow2014generative}  and variational autoencoders (VAEs) ~\cite{kingma2013auto}. These neural networks are trained such that they map a vector residing in a low-dimensional space to a signal in a high-dimensional space with a specific structure, such as sparsity. 
The idea of generative model-based signal reconstruction is conceptually similar to the Bayesian framework because a well-trained generative model acts as a prior to modeling the sparse signal. In this work, we propose a reconstruction algorithm that estimates the unknown sparse signal in the range space of the network by evaluating the low-dimensional vector at the input of the generative model.

A few recent works in one-bit CS employ deep generative models. One notable work considered one-bit CS with rectified linear units-based generative models and provided the analysis of the reconstruction algorithm~\cite{qiu2020robust}. This work also uses the dithering technique, which may not always be practical. On the contrary, our work addresses the general case of $L$-Lipschitz generative models without a dither signal, with results independent of the optimization algorithm. Further, another study investigates the use of a diffusion-based generative model to demonstrate its effectiveness in one-bit CS~\cite{meng2022quantized}. While this score-based generative model yields promising results, its practical applicability is constrained by the substantial size of the training dataset required and the high computational cost associated with training. In addition, signal reconstruction during inference is also computationally intensive, leading to significant latency, particularly on edge devices. This latency is especially problematic for time-sensitive applications. Additionally, the study lacks theoretical analysis of the signal reconstruction performance and the number of measurements required for accurate signal recovery. 

Furthermore, a generative model-based one-bit reconstruction algorithm and its theoretical guarantees are studied from an information-theoretic perspective in \cite{liu2020sample}. Our work complements the study in \cite{liu2020sample}, with several key differences in the setup. Firstly, our framework is rooted in an optimization problem that balances the norm and model mismatch, whereas their approach directly builds on binary iterative hard thresholding (BIHT). Unlike the algorithm proposed in \cite{liu2020sample}, our method avoids projecting the solution onto the range space of the generative models, a non-convex set, thereby simplifying the computations and achieving better performance. Moreover, while both methods guarantee the same order of measurement complexity, our approach offers deeper insights by deriving results from the underlying optimization cost. Specifically, our analysis rigorously accounts for errors arising from solving the non-convex optimization problem and the generative model's limitations in capturing all sparse vectors within its range. In contrast, their results hinge on the assumption that the obtained solution matches the measurements within a specified error margin, an outcome not directly guaranteed by their BIHT-based approach. Moreover, they restrict their analysis to signals within the generator's range that exceed a given norm, which imposes additional constraints on their method's rigor and applicability. The results in ~\cite{liu2020sample} for Gaussian matrices with independent and identically distributed entries were later extended to Gaussian circulant matrices in ~\cite{liu2021robust}.

\subsection{Our Contributions}
This paper addresses the challenge of one-bit CS by integrating model-based methods with deep generative models. The main contributions of this work are as follows:
\begin{itemize}
\item\emph{Algorithm development:} We develop a generative model-based algorithm whose generative part learns the distribution of the sample space of sparse vectors. We then use gradient descent to optimize the representation learned by the model that matches the given measurements.
\item\emph{Theoretical results:} We derive a lower bound on the number of measurements required to ensure that the reconstruction error is bounded. To be specific, we establish that when gradient descent finds a good approximate solution to the optimization problem, the algorithm output is close to the projection of the true sparse vector to the range of the generator (see \Cref{thm:NN_error}). The results are also extended to the noisy measurement cases (see \Cref{cor:noisy,cor:noisy_moremes}).
\item\emph{Empirical validation:} We demonstrate the superior signal reconstruction performance and robustness of our algorithm using mean squared error (MSE) and normalized mean square error (NMSE). We apply it for image compression using publicly available datasets such as MNIST, Fashion-MNIST, and Omniglot. 

The numerical results demonstrate that our algorithm uses fewer measurements than traditional one-bit CS algorithms to offer an improved signal recovery performance. Also, our algorithm is more robust to the noise in one-bit measurements and errors in the measurement matrix compared to traditional algorithms. 
\end{itemize}
Overall, we tackle the problem of one-bit CS by combining the strengths of model-based approaches with the capabilities of deep generative models. This fusion offers excellent recovery performance, enjoys strong theoretical guarantees, and is useful in practical applications.  

A part of this work was published in \cite{joseph2020one}. Beyond \cite{joseph2020one}, in this work, we provide comprehensive proof of the theoretical analysis, including all necessary mathematical tools, as well as extend the results for the noisy cases and discussion comparing our results with existing work. Additionally, we expand our experimental setup to evaluate the robustness of our algorithm in the presence of both additive noise and sign-flip noise. We demonstrate the superior performance of our method under measurement matrix uncertainties and also explore the limitations of the generative model-based approach. Furthermore, we present numerical results using two additional datasets, Fashion MNIST and Omniglot, to assess the effectiveness of our method.

\subsubsection*{Notation} Scalars are represented by lowercase letters and symbols, e.g., $y$ and $\gamma$. Vectors and matrices are denoted by lowercase boldface and uppercase boldface characters, such as $\vecx$ and $\matA$, respectively. Calligraphic letters, such as $\calS$, are used to represent sets. We use $\matA_{i}$ and $\matA_{i,j}$ to denote the $i$th row and $(i,j)$th element of matrix $\matA$, respectively. The transpose of a vector $\vecx$ is represented by $\vecx\tran$ and its $\ell_2$ norm is written as $\| \vecx\|$. The cardinality of a set $\calS$ is denoted as $|\calS|$. A closed interval between two points $a$ and $b$ is given by $[a,b]$. The operator $\sign\lb p\rb$ represents the sign operation, which is defined as $\sign\lb p\rb = +1$, if $p> 0$ and $-1$, if $p\leq 0.$ The set of real numbers is denoted by $\mathbb{R}$.  We use $\calB^s_r=\lc\vecz\in\bbR^s:\lV\vecz\rV\leq r\rc$ to denote a ball of radius $r$ in~$\bbR^s$. Finally, for a set $\calT \subseteq \bbR^s$, and a function $G: \bbR^s\to\bbR^n$, we write $G(\calT) = \{G(\vecz): \vecz \in \calT\}$.

\section{System Model and the Generative Model-Based Algorithm}
We consider the problem of recovering an unknown sparse vector $\vecx^*\in\bbR^n$ from a set of one-bit measurements $\vecy\in\lc\pm1\rc^m$,  modeled as follows:
\begin{equation}\label{eq:model}
\vecy_i=\sign\lb\matA\vecx^*\rb \in\lc\pm1\rc^m
\end{equation}
where $\matA\in\bbR^{m\times n}$ is the known measurement matrix. Our goal is to find a reconstruction $\hat{\vecx}$ is such that it is close to the original signal $\vecx^*$.

Our approach is to use a generative model which is given by a deterministic function: $G:\bbR^s\to\bbR^n$. One option for the generative model is a feedforward neural network with $d$ layers given by
\begin{equation}
G(\vecz) = \phi_d \big(\phi_{d-1} (\cdots \phi_2(\phi_1(\vecz, \bs{\theta}^{(1)}), \bs{\theta}^{(2)})\ldots, \bs{\theta}^{(d-1)}), \bs{\theta}^{(d)}\big),
\end{equation}
where $\vecz \in \mathbb{R}^s$ is the latent variable, $\phi_i(\cdot)$ is the functional mapping corresponding to the $i$-th layer of the neural network, and $\bs{\theta}^{(i)} = (\bs{W}^{(i)}, \bs{b}^{(i)})$ represents the parameter pair for the $i$-th layer, which has $N_i$ nodes at its output. Here, $\bs{W}^{(i)} \in \mathbb{R}^{N_i \times N_{i-1}}$ is the weight matrix, and $\bs{b}^{(i)} \in \mathbb{R}^{N_i}$ is the bias vector. Note that $N_0 = s$ and $N_d = n$. 
Further, let $\bs{z}^{(i)}$ denote the output of the $i$-th layer, defined as $\bs{z}^{(i)} = \phi_i(\vecz^{(i-1)}, \bs{\theta}^{(i)}) = \tilde{\phi}_i(\bs{W}^{(i)} \vecz^{(i-1)} + \bs{b}^{(i)})$, where $\tilde{\phi}_i$ is the element-wise non-linear activation function used in the $i$-th layer. Common choices for $\tilde{\phi}_i$ include the ReLU function, sigmoid function, and hyperbolic tangent function.

We assume that the latent variable $\vecz$ follows a fixed distribution $p_Z$ over $\mathbb{R}^s$. During the training phase, the algorithm learns the function $G$ that maps the distribution $p_Z$ to the data distribution using the training samples. It is important to note that the training process does not involve one-bit measurements, nor does the model aim to learn how to invert the one-bit measurement function. Instead, it focuses on efficiently learning and representing the signals of interest. Specifically, the generative model is trained with sparse vectors so that the range of the generator, denoted by $\calS$, closely approximates the desired set of sparse vectors. Since the set of sparse vectors in $\mathbb{R}^n$ forms a small subset of $\bbR^n$, we choose $s \ll n$, meaning $G$ is a mapping from a low-dimensional representation space ($\mathbb{R}^s$) to a high-dimensional sample space ($\mathbb{R}^n$), which is learned by the model. 

Once we train the generator, we use it to recover the unknown sparse vector $\vecx^*$ from a set of one-bit measurements $\vecy$ by using $\calS$ as an approximation for the set of sparse vectors. To this end, we minimize the following objective function, which depends on $G(\vecz)$,
\begin{equation}\label{eq:cost}
l_{\mathrm{loss}}(G(\vecz))=\lV G(\vecz)\rV^2-\frac{\sqrt{2\pi}}{m}\vecy\tran\matA G(\vecz).
\end{equation}
The second term of the objective function maximizes the correlation between the one-bit measurements $\vecy$ and the corresponding linear measurements. For a fixed $l_2$ norm of $G(\vecz)$, the term is maximized when $\sign(\matA G(\vecz))=\vecy$. Therefore, the second term ensures the match between $\matA G(\vecz)$ and $\vecy$. However, the second term decreases as the $l_2$ norm of $G(\vecz)$ increases, and therefore, we use the first term to control the norm. Hence, the two terms of the objective function jointly optimize the representation error.  We denote
\begin{equation}\label{eq:optimization}
    \hat{\vecz} = \underset{\vecz\in\bbR^s}{\arg\min}\; l_{\mathrm{loss}}(G(\vecz)).
\end{equation}
While any optimization procedure can be used to minimize the loss function, we use the standard back-propagation algorithm to reconstruct the compressed signal. Let $\hat{\vecz}$ denote the optimization procedure output. The reconstructed signal is given by $\vecx=G(\hat{\vecz})$. 
The objective function to be optimized is non-convex, but we still use the gradient descent algorithm to solve the optimization problem. The gradient descent algorithm is expected to provide a locally optimum solution which we assume to be a good approximate solution. We corroborate this assumption empirically (as demonstrated in \Cref{sec:sim}), to show that the gradient descent solution achieves good reconstruction performance.

We conclude this section by highlighting the main advantages of the generative model-based optimization compared to classical optimization-based algorithms, such as the convex optimization-based algorithm~\cite{plan2012robust}. Firstly, the use of the generator's range in the cost function automatically eliminates the need to constrain the solution to sparse vectors. This is because our approach ensures that the obtained solution $\vecx\in\calS$, which approximates the set of sparse vectors. Quantifying non-convex sparsity constraints is a significant challenge, and traditional approaches often resort to approximations like the $\ell_1$ norm or non-convex $\ell_p$ norms (which complicate the problem). Secondly, the objective function in \eqref{eq:optimization} is minimized with respect to the latent variable $\vecz\in\bbR^s$ rather than the unknown vector $\vecx\in\bbR^n$. This results in a smaller search space, as $s \ll n$, and reduces the computational and memory complexity of the algorithm. This is critical, especially because one-bit CS is used in resource-constrained systems. Finally, unlike model-based algorithms, which are general and can only capture sparsity in the signal, the generator model can learn any additional structure in the desired set of sparse vectors that arise due to the underlying physics or system properties (e.g., images with pixel values constrained between 0 and 255). This additional knowledge can further reduce the size of the latent space, modeled by $s$, thereby decreasing both the training and reconstruction complexity of the algorithm, which depends on $s$. Moreover, the model is highly flexible and can be combined with any other cost functions and measurement models, such as the $B$-bit quantized measurement model.

\section{Theoretical Analysis: Measurement Bound}\label{sec:theory}
Our neural network-based generative model-based algorithm combines the strengths of model-based algorithms and data-driven models by overcoming some of the disadvantages of model-based approaches, as discussed above, while also retaining the strong theoretical foundations of model-based algorithms. Next, we discuss some theoretical guarantees for the algorithm above, assuming that gradient descent finds a good approximate solution to the non-convex optimization problem in \eqref{eq:optimization}.

Our analysis makes the following assumptions:
\begin{itemize}[leftmargin=0cm]
    \item [] \emph{Assumption 1:} The generator function $G$ is a $d-$layer neural network with at most $N=\max_i\;N_i$ nodes per layer where all weights are upper bounded by  $ w_{\max}$ in absolute value, and the non-linearity after each layer is $L-$Lipschitz.
    \item [] \emph{Assumption 2:} The input $\vecz$ to the model $G$ have independent entries drawn from a  uniform distribution over $\ls-\frac{r}{\sqrt{s}},\frac{r}{\sqrt{s}}\rs$ during the training phase.
    \item [] \emph{Assumption 3:} The measurement matrix   $\matA\in\bbR^{m\times n}$ is a Gaussian random matrix with independent and identically distributed entries, $\matA_{i,j}\sim\calN(0,1/m)$.
\end{itemize}
Based on the above assumptions, the main result of this section is as follows:
\begin{theorem}\label{thm:NN_error}
Suppose that Assumptions 1-3 hold, the ground truth $\vecx^*$ is such that it satisfies $\lV\vecx^*\rV = 1$, and the measurement vector $\vecy$ follows the model given by \eqref{eq:model}. Suppose $\tilde{\vecz}$ minimizes the cost function in \eqref{eq:cost} to within additive $\delta$ of the optimum over the vectors with $\lV \vecz\rV\leq r$. Then, for any $\epsilon>0$, there exist universal constants $C,c>0$ such that if \begin{equation}\label{eq:m_bound}
m\geq C\epsilon^{-2}s\lb r^2+d\log LNw_{\max} \rb,
\end{equation} 
with probability at least $1-4\exp\lb-c\epsilon^2 m\rb$, the following holds,
\begin{equation}\label{eq:l_bound}
\lV G(\tilde{\vecz})-\vecx^*\rV^2\leq \underset{\substack{\vecz\in\bbR^s\\\lV\vecz\rV\leq r}}{\min}\lV G(\vecz)-\vecx^*\rV^2+\delta+\epsilon.
\end{equation}
\end{theorem}
\begin{proof}
    See~\Cref{app:NN_error}.
\end{proof}

Before examining insights from the algorithm, we first discuss the intuition and justification behind the assumptions.  Assumption 1 states that all activation functions in the generative model are Lipschitz. This is reasonable since commonly used activation functions, such as ReLU, sigmoid, and hyperbolic tangent, are 1-Lipschitz. This assumption is both practical and flexible, allowing a variety of neural network designs.  The second assumption requires the input $\vecz$ to follow a uniform distribution. Although the assumption looks stringent, the proof only requires $\vecz$ to be bounded. While Assumption 2 ensures that $\lV\vecz\rV\leq s\lV\vecz\rV_{\infty}=r$, any bounded distribution or light-tailed distribution (such as those with exponential decay) satisfies this requirement with a high probability for a suitable choice of $r$. This makes our framework more general and the assumption non-restrictive.  Finally, Assumption 3 enforces that the measurement matrix is Gaussian, which is standard in CS. It is necessary to derive the concentration inequality using Gaussian mean width, a well-established tool in the literature.  

\subsection{Discussion}
We now discuss insights from the main result. Our error bound has three terms, the first being representation error. It arises because the generator’s range may not perfectly match the set of sparse vectors, making the exact representation of $\vecx^* \in \calS$ impossible. Thus, $\lV G(\vecz) - \vecx^* \rV^2$ may not be 0 for any $\vecz$ with $\lV\vecz\rV \leq r$, and the first term in \eqref{eq:l_bound} is the best achievable error under imperfect training. The second term $\delta$ accounts for the fact that the gradient descent does not necessarily converge to the global optimum. We note that $\delta$ is the difference in the cost function and the optimum optimization variable, i.e., 
\begin{equation}
    l_{\mathrm{loss}}(\tilde{\vecz})\leq \underset{\substack{\vecz\in\bbR^s,\lV\vecz\rV\leq r\\\lV G(\vecz)\rV\leq 1}}{\min}l_{\mathrm{loss}}(\vecz)+ \delta.
\end{equation}
Finally, the error term $\epsilon$ can be controlled by adjusting the number of measurements, and it can be driven close to zero. Empirically, we observe that the reconstruction error converges to zero, leading to the conclusion that with appropriate training, the generator’s range can effectively approximate the set of sparse vectors, and the gradient descent algorithm yields a good solution.

Next, we discuss the measurement bound and its dependence on the latent distribution and network parameters. As $r$ increases, the number of required measurements grows, while the estimation error decreases. This is captured by the first term in \eqref{eq:l_bound} that monotonically decreases with $r$. This observation is intuitive because as $r$ increases, the generator’s domain expands, so the search space for optimization in \eqref{eq:optimization} expands, requiring more measurements for resolvability. Additionally, as the domain expands, the range expands, leading to less representation error, and thereby improved accuracy. Similarly, as the network parameters, $s,d,N,L$, and $w_{\max}$ increase, the number of required measurements increases. The reason is that as these parameters increase, the network becomes more flexible, thus allowing the range of $G(\vecz)$ to expand. Consequently, the first term in the error bound decreases. Hence, an increase in the number of required measurements results in an improved error bound, as expected. Further, it is interesting to note that, unlike the traditional CS guarantees, the measurement bound does not directly depend on the dimension of the unknown sparse vector $n$. However, the dependence on $n$ is implicitly captured by the term $N$ as 
\begin{equation}
N=\max_i\;N_i\geq N_d=n.    
\end{equation}
Therefore, similar to the traditional CS guarantees, here $m$ grows logarithmically with $n$ if we choose $N=\Omega(n)$. 

\subsection{Extension to Noisy Measurement Model}
Our results can be extended to the noisy measurement case. So we consider the noisy measurements given by
\begin{equation}\label{eq:model_corrupt}
    \vecy=\bs{\eta}\odot\sign(\matA\vecx^*),
\end{equation}
where $\odot$ is the Hadamard product and $\bs{\eta}\in\{\pm1\}^m$ represents the corruption. If $\bs{\eta}_i=1$, there is no corruption in $\vecy_i$, and if $\bs{\eta}_i=-1$, the sign of $\vecy_i$ is flipped. 
\begin{cor}\label{cor:noisy}
Suppose that Assumptions 1-3 hold and the ground truth $\vecx^*$ is such that it satisfies $\lV\vecx^*\rV = 1$. Let the measurement vector $\vecy$ follow the model given by \eqref{eq:model_corrupt} where entries of $\bs{\eta}$ are independent and follow a Rademacher distribution satisfying $\bbP\{\bs{\eta}_i=1\}=\alpha\in(0.5,1]$. Suppose $\tilde{\vecz}$ minimizes the cost function in \eqref{eq:cost} to within additive $\delta$ of the optimum over the vectors with $\lV \vecz\rV\leq r$ and $\lV G(\vecz)\rV\leq 1$. Then, for any $\epsilon>0$ there exist universal constants $C,c>0$ such that if \eqref{eq:m_bound} holds,
with probability at least $1-4\exp\lb-c\epsilon^2 m\rb$
\begin{multline}\label{eq:l_bound_noisy}
\lV G(\tilde{\vecz})-\vecx^*\rV^2\leq\underset{\substack{\vecz\in\bbR^s,\lV\vecz\rV\leq r\\\lV G(\vecz)\rV\leq 1}}{\min}\lV G(\vecz)-\vecx^*\rV^2\\+\frac{\delta+\epsilon}{2\alpha-1}+\frac{2(1-\alpha)}{2\alpha-1}.
\end{multline}
\end{cor}
\begin{proof}
See \Cref{app:noisy}.
\end{proof}
We note that the noisy measurements introduce an additional error term given by $ \frac{2(1-\alpha)}{2\alpha-1} $, which vanishes as $ \alpha$ approaches $ 1 $ and increases the $\delta+\epsilon$ error by a factor of $ \frac{1}{2\alpha-1} $. Moreover, as $\alpha $ approaches $ 0.5 $, the error diverges to infinity.  Also, the bound $ \lV G(\tilde{\vecz}) - \vecx^* \rV^2 \leq 2 $ implies that the theoretical bound remains useful only if $\alpha > 2/3 $, ensuring that $ \frac{2(1-\alpha)}{2\alpha-1} < 2 $. Therefore, to theoretically guarantee a successful recovery, we require more than $2/3 $ correct measurements on average. Additionally, the above results imposes an extra condition that $\lV G(\vecz)\rV\leq 1$. However, if the noise parameter $\alpha$ is known to the algorithm, we can modify the cost function to incorporate this information. This modification removes the bound on $\lV G(\vecz)\rV$ and eliminates the extra bias term $ \frac{2(1-\alpha)}{2\alpha-1} $ in the error, implying $\alpha$ can be less than $2/3$.
\begin{cor}\label{cor:noisy_moremes}
Suppose that Assumptions 1-3 hold and the ground truth $\vecx^*$ is such that it satisfies $\lV\vecx^*\rV = 1$. Let the measurement vector $\vecy$ follow the model given by \eqref{eq:model_corrupt} where entries of $\bs{\eta}$ are independent and follow a Rademacher distribution satisfying $\bbP\{\bs{\eta}_i=1\}=\alpha\in(0.5,1]$. Suppose $\tilde{\vecz}$ minimizes the following to within additive $\bar{\delta}$ of the optimum over the vectors with $\lV \vecz\rV\leq r$,
\begin{equation}
   \bar{l}_{\mathrm{loss}}(G(\vecz)) = \lV G(\vecz)\rV^2-\frac{\sqrt{2\pi}}{m(2\alpha-1)}\vecy\tran\matA G(\vecz).
\end{equation}
 Then, for any $\epsilon>0$ there exist universal constants $C,c>0$ such that if \eqref{eq:m_bound} holds, with probability at least $1-4\exp\lb-c\epsilon^2 m\rb$, 
\begin{equation}\label{eq:l_bound_noisy_mod}
\lV G(\tilde{\vecz})-\vecx^*\rV^2\leq \underset{\substack{\vecz\in\bbR^s\\\lV\vecz\rV\leq r}}{\min}\lV G(\vecz)-\vecx^*\rV^2+\bar{\delta}+\frac{\epsilon}{2\alpha-1}.
\end{equation}
\end{cor}
\begin{proof}
    See~\Cref{app:noisy_moremes}.
\end{proof}
Here, the effect of $\alpha$ on the error is captured through $\frac{\epsilon}{2\alpha-1}$ and $\bar{\delta}$ as it represents the additive error with respect to the new cost function. For example, for $\vecz$ and $\tilde{\vecz}$ such that $\lV G(\vecz)\rV^2=\lV G(\tilde{\vecz})\rV^2$, we derive
\begin{align}
   \bar{\delta}&\leq \bar{l}_{\mathrm{loss}}(G(\vecz))- \bar{l}_{\mathrm{loss}}(G(\tilde{\vecz}))\\
   &=\frac{1}{2\alpha-1}[l_{\mathrm{loss}}(G(\vecz))- l_{\mathrm{loss}}(G(\tilde{\vecz}))].
\end{align}
Therefore, we see that $\bar{\delta}=\Omega(\delta/(2\alpha-1))$, implying the corresponding error term increases by a factor of $1/(2\alpha-1)$.

\subsection{Comparison With the Results in \cite{liu2020sample}}
The result in \cite{liu2020sample} uses two distance metrics. For any two vectors $\vecu,\vecv\in\bbR^M$, we define $d_{\mathrm{H}}(\vecu,\vecv)$ as the Hamming distance, and $d_{\mathrm{s}}$ as the geodesic distance, which is the normalized angle
between the vectors.
\begin{align}
\label{eq:Hamming}
   d_{\mathrm{H}}(\vecu,\vecv) =\frac{1}{M}\sum_{i=1}^M\vecu_i\neq \vecv_i\\
d_{\mathrm{s}}(\vecu,\vecv) =\frac{1}{\pi} \arccos\lb\frac{\vecu\tran\vecv}{\|\vecu\|_2\|\vecv\|_2} \rb.
\end{align}
We next state the result from  \cite{liu2020sample}.
\begin{prop}[{\cite[Theorem 4]{liu2020sample}}]\label{prop:liu_result}
Suppose Assumptions~1-3 hold with $L=1$. Also, for any $\vecx^* \in G(\mathcal{B}_r^2) \setminus B_{R_{\min}}^n$,  let $\vecy\in\bbR^{m}$ be any corrupted measurements satisfying $d_{\mathrm{H}}(\vecy, \sign(\matA\vecx^*)) \leq \tau_1$. Then, for fixed $\epsilon \in (0, 1)$ and $R_{\min} > 0$, if
\begin{equation}
m = \Omega\left( s \epsilon^{-2} \log \frac{r (N w_{\max})^d}{ R_{\min}\epsilon} \right),
\end{equation}
with probability at least $1 - e^{-\Omega(\epsilon^2 m)}$, any $\hat{\vecx} \in G(\mathcal{B}_r^2) \setminus B_{R_{\min}}^n$ with $ d_{\mathrm{H}}(\sign(\matA\hat{\vecx}), \vecy) \le \tau_2$ satisfies
\begin{equation}
d_{\mathrm{s}}(\hat{\vecx},\vecx^*) \leq \epsilon + \tau_1 + \tau_2
\end{equation} 
\end{prop}

Comparing \Cref{thm:NN_error} and \Cref{prop:liu_result}, we first note that both results exhibit similar order complexity. Specifically, when $m \approx \Omega(1/\epsilon^2)$, the error scales on the order of $\epsilon$ with probability at least $1 - e^{-\Omega(\epsilon^2 m)}$. The key differences lie in the underlying assumptions and approach.

Firstly, \Cref{prop:liu_result} is not tied to a specific algorithm but merely guarantees that if an algorithm can find a solution within the range of the generator such that  
$d_{\mathrm{H}}(\sign(\matA\hat{\vecx}), \vecy) \leq \tau_2$,   
then the bound holds. However, it does not establish whether such a solution exists or how to obtain it. The algorithm mentioned in the paper is an adaptation of BIHT, which is not explicitly designed to achieve this condition through its cost function or formulation. In contrast, our result in \Cref{thm:NN_error} is derived from the optimal point of our specific cost function. Additionally, we model the possibility of reaching a non-optimal solution to our optimization problem through the term $\delta$.

Secondly, \Cref{prop:liu_result} introduces an additional parameter $R_{\min} > 0$, requiring both the true vector and the recovered vector to have a norm greater than $R_{\min}$. Moreover, the measurement bound increases logarithmically with $1/R_{\min}$, meaning that if $R_{\min}$ approaches zero, the required number of measurements becomes impractically large. Conversely, our result in \Cref{thm:NN_error} imposes no such constraint on the ground truth or recovered solution. Additionally, the set of vectors with a norm greater than $R_{\min}$ is highly non-convex, making the constraint difficult to satisfy. Enforcing this constraint in BIHT-based methods through projection is non-trivial and is not discussed in the paper. Our result does not assume any minimum norm constraint, allowing us to conceptually set $R_{\min} = 0$ without inflating the measurement bound.  

Furthermore, our result defines the reconstruction error in terms of the $\ell_2$-norm, whereas \Cref{prop:liu_result} employs the geodesic distance. These metrics are equivalent if $G(\tilde{z})$ has a unit norm, but neither result assumes this condition. However, since the measurements lack amplitude information, we are primarily concerned with direction rather than magnitude. Thus, normalizing the solution ensures both error bounds remain comparable.  

Additionally, the noisy measurement models used in our result and \Cref{prop:liu_result} differ. \Cref{prop:liu_result} assumes an adversarial noise model, where the number of sign flips is known, whereas we consider a random noise model, where only the noise statistics are known. Both models are well-studied in the literature, but ours makes fewer assumptions. Naturally, our model has a greater impact on the error term, scaling it as $\delta + \epsilon$, whereas theirs introduces only an additive term. In the absence of noise, both bounds scale similarly.  

Finally, the dependence on error term $\epsilon$ differs slightly between the two results. In \Cref{prop:liu_result}, the required number of measurements is  
$m = \Omega(\epsilon^2 \log(1/\epsilon)),$  
whereas our result provides a tighter bound of $m = \Omega(\epsilon^2)$. However, in our case, $m = \Omega(r^2)$, while \Cref{prop:liu_result} states $m = \Omega(\log r)$. Since $r$ is a fixed parameter, it does not influence the measurement bound in the same way as $\epsilon$ that directly controls the error. The dependence on other parameters such as $d,N$, and $w_{\max}$ is similar in both results. Although \Cref{prop:liu_result} assumes $L = 1$, extending it to arbitrary $L$ is straightforward, yielding a result comparable to \Cref{thm:NN_error} and maintaining similar dependence of $m$ on $L$.  

\section{Simulation Results}\label{sec:sim}
In this section, we evaluate the signal reconstruction performance of our algorithm using three publicly available datasets, the MNIST handwritten digit dataset~\cite{mnist2010LeCun}, the Fashion MNIST dataset~\cite{fashionMnist}, and the Omniglot dataset ~\cite{lake2015human}. These image datasets are sparse in the pixel intensities. We compare the performance of our algorithm against two traditional one-bit CS algorithms: the convex optimization-based algorithm (labeled as \texttt{YP})~\cite{plan2012robust}, and BIHT (labeled as \texttt{BIHT})~\cite{binarystableEmbedd} algorithm. 

For the generative models, we follow the setup from~\cite{bora2017compressed} and train VAEs~\cite{kingma2013auto} as the generative model using training images from the MNIST, Fashion-MNIST, and Omniglot datasets. The image size is $28\times 28$, resulting in an input dimension of $N=784$. We choose the input size for the generator as $s = 40$. The generator is implemented as a fully connected neural network with layers of size 40-500-500-746, while the encoder consists of layers of size 784-500-500-40. VAEs for each dataset are trained for 200 epochs with a mini-batch size of 64 using the Adam optimizer~\cite{kingma2014adam} with a learning rate of 0.001. 

To assess the reconstruction performance, we use noisy one-bit compressed measurements as
\begin{equation} \label{eq:noisymeasurements}
\bs{y}= \bs{\eta}\odot\sign(\matA\vecx^* + \bs{n} ) \in\lc\pm1\rc^m,    
\end{equation}
 where the entries of $\bs{n}$ are independent Gaussian random variables with mean zero and variance $v_n$ and entries of $\bs{\eta}$ are independent Rademacher random variables that takes values $1$ and $-1$ with probability $\alpha$ and $1-\alpha$, respectively. The columns of the measurement matrix $\matA$  are drawn uniformly from the surface of the $m$-dimensional unit hypersphere~\cite{muller1959note}.

To quantify reconstruction performance, we use the MSE and NMSE metrics, which are defined as
\begin{equation}        
\mathrm{MSE} = \|\vecx^* -  \widehat{\vecx} \|^2\quad\text{and}\quad
        \mathrm{NMSE} = \Bigg\|\frac{\vecx^*}{ \| \vecx^*\|} -  \frac{\widehat{\vecx}}{ \| \widehat{\vecx} \|}  \Bigg\|^2,
  \end{equation}
 where $\vecx^*$ and $ \widehat{\vecx}$ represent the true and the estimated signals, respectively. We generate one-bit compressed measurements for 10 images from the testing set using \eqref{eq:noisymeasurements}. Given the non-convex nature of the optimization problem defined in \eqref{eq:cost}, we perform 10 random restarts with 100 gradient descent steps per restart for signal reconstruction and report the result with the least error. For the noisy case, we average the results over 50 Monte Carlo runs for each of the ten images. 
 
 \begin{figure}
\subfloat[MSE performance]{
\hspace{-0.5cm}\label{fig:MSE}
\includegraphics[width=\linewidth/2]{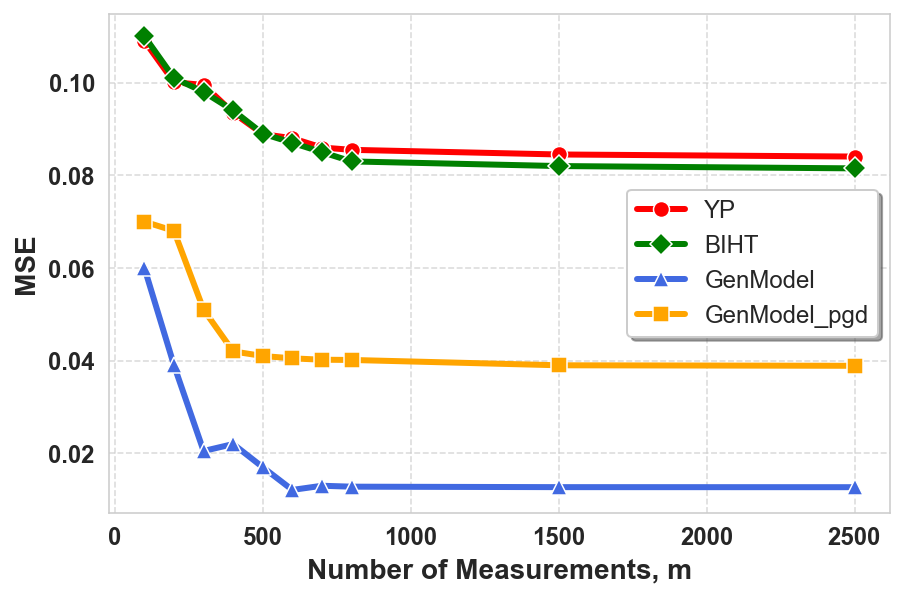}
}
\subfloat[NMSE performance]{\hspace{-0.3cm}
\label{fig:NMSE}
\includegraphics[width=\linewidth/2]{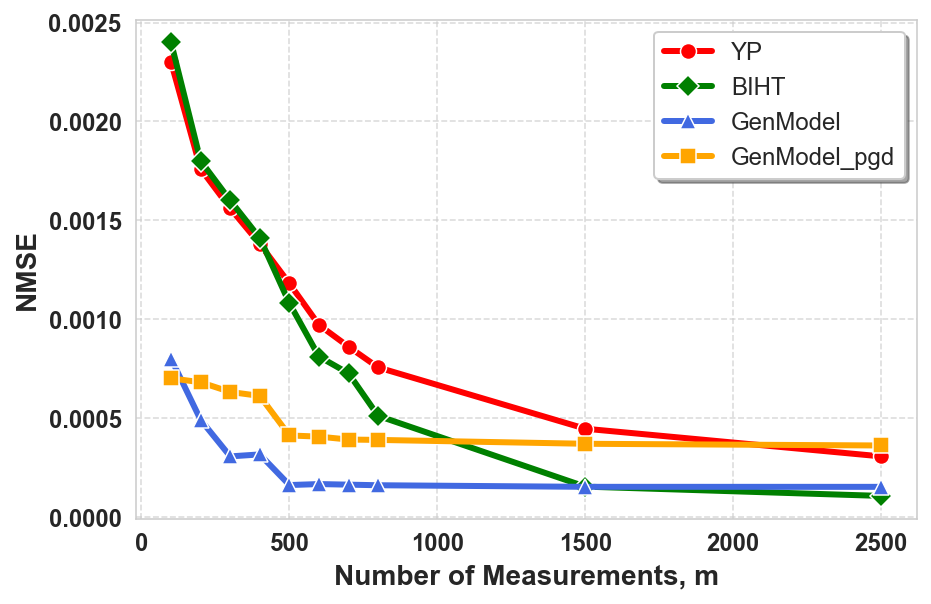}
}

\caption{Reconstruction performance of our algorithm compared with \texttt{YP}~\cite{plan2012robust}, \texttt{BIHT}~\cite{binarystableEmbedd}, and  \texttt{GenModel\_pgd}~\cite{liu2020sample} as a function of number of measurements $m$ in the noiseless setting.}
\label{fig:error}
\end{figure}

\begin{figure}[tb!]
\subfloat[MSE performance]{
\hspace{-0.5cm}\label{fig:MSE_noisy}
\includegraphics[width=\linewidth/2]{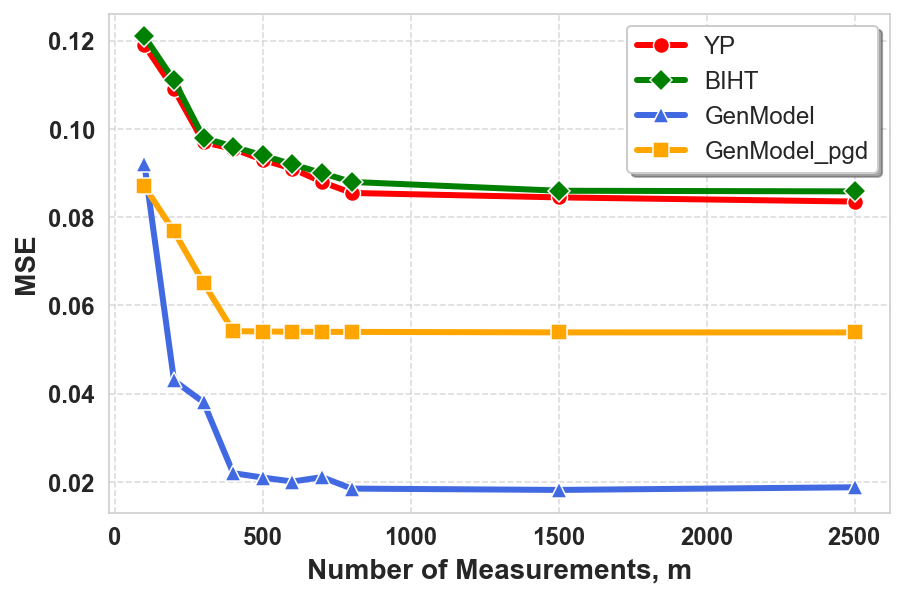}
}
\subfloat[NMSE performance]{\hspace{-0.3cm}
\label{fig:n1_NMSE_noisy}
\includegraphics[width=\linewidth/2]{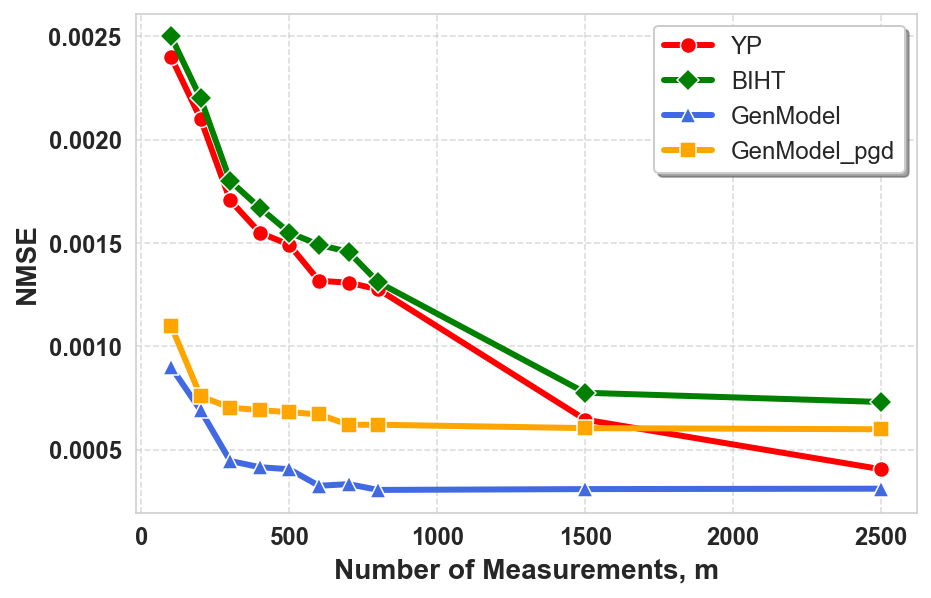}
}
\caption{Reconstruction performance of our algorithm, \texttt{GenModel\_pgd}, \texttt{BIHT}, and \texttt{YP} as a function of number of measurements $m$ when $v_n = 0.1$ and $\alpha = 0.85$.}
		\label{fig:noisyerror}
\end{figure}

\subsection{Noiseless Setting}
In \Cref{fig:error}, we present the recovery performance of our algorithm (labeled as \texttt{GenModel})  alongside the generative model-based one-bit CS algorithm in  \cite{liu2020sample} (labeled as \texttt{GenModel\_pgd}), and the traditional algorithms \texttt{YP} and \texttt{BIHT} in a noiseless scenario, i.e., when $v_n$ = 0 and $\alpha$ = 1.  As the number of measurements $m$ increases, the performance of all algorithms improves, as more information about the sparse vector becomes available.
 However, distinct performance traits of our algorithm are evident when examining the MSE and the NMSE metrics:
\subsubsection{MSE Performance} 
The generative model-based algorithms, \texttt{GenModel} and \texttt{GenModel\_pgd}, significantly outperform traditional algorithms in terms of MSE, achieving an order-of-magnitude improvement for the same $m$. This superior performance arises because generative models are capable of learning the distribution of compressed signals. When well-trained, they recover both the magnitude and direction of the compressed signal. For the same reason, \texttt{GenModel\_pgd} also performs better than the traditional algorithms. Also, \texttt{BIHT} and \texttt{YP} estimate the signals on the unit ball, limiting their accuracy, especially when the sparse signal does not lie on the unit ball.  Further, we compare the two generative model-based algorithms, our \texttt{GenModel}, which is based on a well-defined optimization problem, and the BIHT-based approach in \texttt{GenModel\_pgd}. Our algorithm clearly outperforms \texttt{GenModel\_pgd} in terms of MSE, indicating a more accurate estimation of the sparse vector.

\subsubsection{NMSE Performance}
The NMSE performance of our algorithm \texttt{GenModel} exhibits two distinct regimes. We first look at the regime with a small number of measurements, $m<1500$.  In this regime, our algorithm achieves superior reconstruction performance compared to traditional algorithms. This is because the range space of a well-trained generative model has more structural information about the signal beyond sparsity, serving as a strong prior for the compressed signal. Our algorithm leverages this well-trained generative model to compute more accurate estimates of the direction of the compressed signal using fewer one-bit measurements.  Next, we look at the regime with a large number of measurements, $m\geq1500$. When the number of measurements is sufficiently large, the traditional algorithms either match or surpass the performance of our algorithm. In \Cref{fig:NMSE}, the \texttt{BIHT} algorithm outperforms our algorithm when $m \geq 1500.$ It can be observed that the NMSE value of our algorithm stagnates after a certain value of $m$. This is because our algorithm is confined to the range space of the generative model to compute an estimate of the sparse signal, i.e., the NMSE of our algorithm is always lower-bounded by the representation error (similar to the first term in \eqref{eq:l_bound} of \Cref{thm:NN_error}). 
In contrast, the traditional algorithms continue to benefit from additional measurements, resulting in a monotonic decrease in NMSE values with an increase in $m$. 
It is also worth noting that while the generative model-based reconstruction algorithm \texttt{GenModel\_pgd} follows a similar performance trend as our algorithm, it demonstrates poorer reconstruction performance.

\begin{figure}[hptb!]
\subfloat[MSE performance]{
\hspace{-0.5cm}\label{fig:signflip}
\includegraphics[width=\linewidth/2]{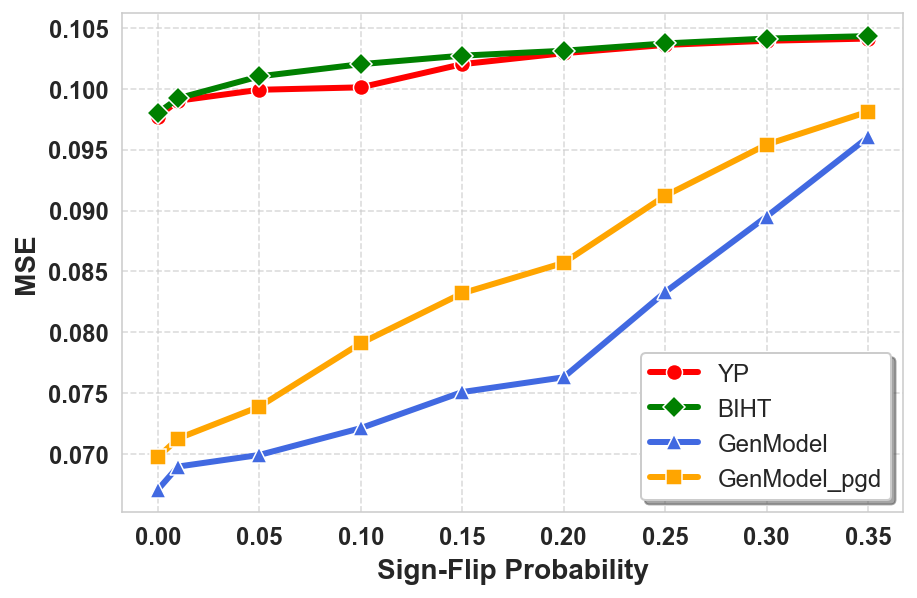}
}
\subfloat[NMSE performance]{\hspace{-0.3cm}
\label{fig:n2_NMSE}
\includegraphics[width=\linewidth/2]{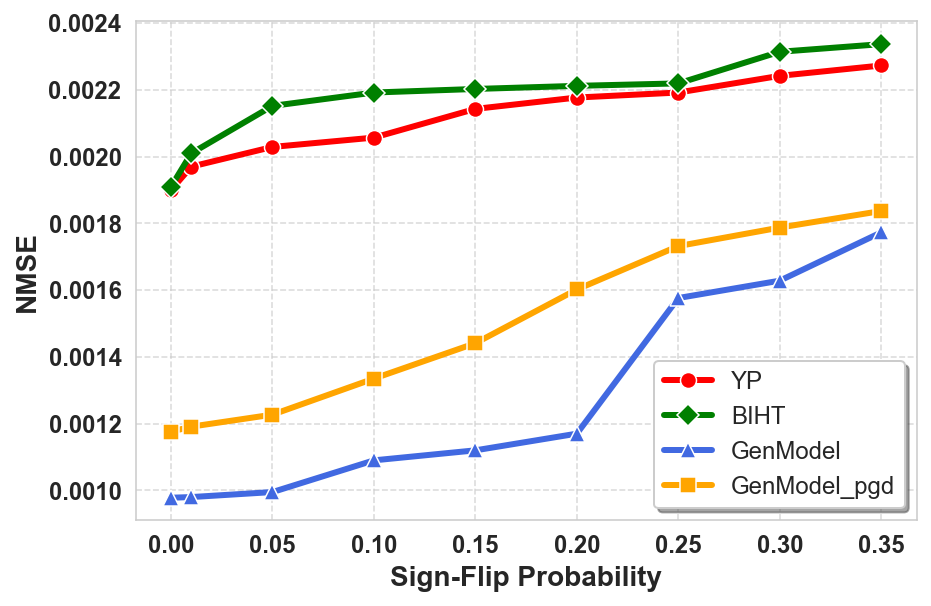}
}
\caption{Reconstruction performance of our algorithm, \texttt{GenModel\_pgd}, \texttt{BIHT}, and \texttt{YP} as a function of sign-flip probability when $m=784$ and $v_n = 0.1$.}
\label{fig:signerror}
\end{figure}

\subsection{Noisy Setting}
In this subsection, we compare the robustness of our algorithm and the traditional algorithms in the presence of  uncertainties in the measurements and measurement matrix. We consider three settings as described below:

\subsubsection{Additive Noise in Measurements} 
In \Cref{fig:noisyerror}, we provide a comparison of the recovery performance of our algorithm with \texttt{GenModel\_pgd}, \texttt{YP}, and \texttt{BIHT} when $v_n$ = 0.1 and $1-\alpha$ = 0.15. We can observe similar trends in the recovery performance of our algorithm, \texttt{GenModel\_pgd}, \texttt{BIHT}, and \texttt{YP} as in \Cref{fig:error}, but with a few notable differences. First, the reconstruction performances of all these algorithms have degraded with noise compared to the results in \Cref{fig:error}. Second, unlike in \Cref{fig:error}, the NMSE performance of \texttt{GenModel\_pgd}, \texttt{BIHT},  or  \texttt{YP} never matches or surpasses the performance of our algorithm for any value of $m$. This indicates that the performance of the traditional algorithms is more sensitive to noise. 
 
Hence, our algorithm outperforms the traditional algorithms in terms of both MSE and NMSE metrics in the noisy setup. 

\subsubsection{Sign Flips in Measurements}
In \Cref{fig:signerror}, we plot MSE and NMSE values for our algorithm, \texttt{GenModel\_pgd}, \texttt{BIHT}, and \texttt{YP} as a function of sign-flip probability $1-\alpha$ when $m = 784$, and $v_n=0.1$. These results further confirm that the proposed method has better reconstruction performance in the presence of sign-flip noise compared to both generative model-based algorithm \texttt{GenModel\_pgd} and traditional algorithms \texttt{BIHT}, and \texttt{YP}.
In \Cref{fig:signerror}, we plot MSE and NMSE values for our algorithm, \texttt{GenModel\_pgd}, \texttt{BIHT} and \texttt{YP} as a function of sign-flip probability $1-\alpha$ when $m = 784$, and $v_n=0.1$. These results further confirm that the proposed method has better reconstruction performance in the presence of sign-flip noise compared to both the generative model-based algorithm  \texttt{GenModel\_pgd} and the traditional algorithms \texttt{BIHT} and \texttt{YP}.
\begin{figure}
\centering
\includegraphics[ width=\linewidth/2]{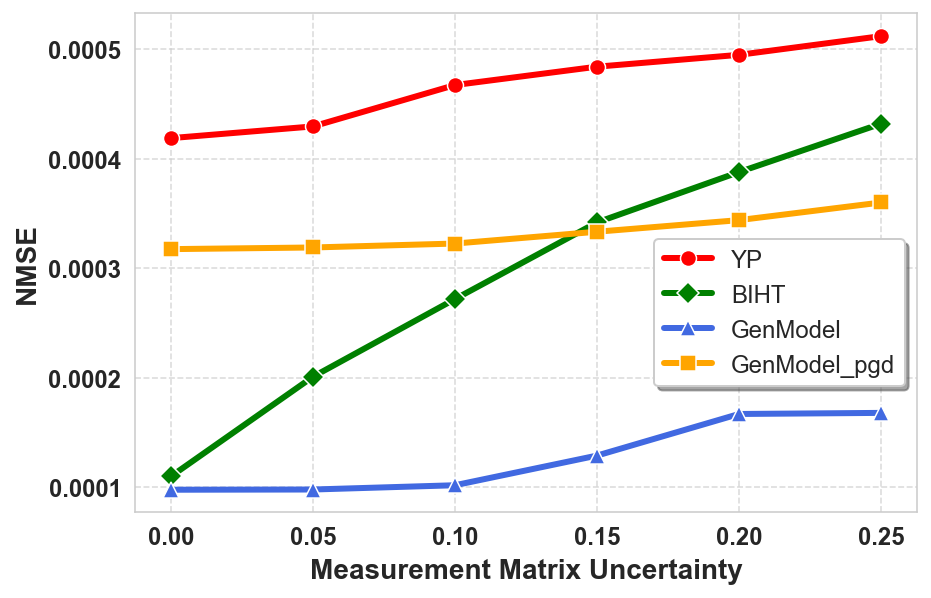}
\caption{Reconstructions performance of our algorithm, \texttt{GenModel\_pgd}, \texttt{BIHT}, and \texttt{YP} as a function of measurement matrix uncertainty, $v_\Delta$, when $m = 1500$, $\alpha=1$, and,  $v_n = 0$.}
\label{fig:uncertainties}
\end{figure}

\subsubsection{Measurement Matrix Uncertainty}
Let  $\matA^\prime$ be the perturbed measurement matrix given to the algorithm, which is given by
$\matA^\prime =  \matA + \Delta$, where $\Delta$ is the unknown perturbation in the measurement matrix. We draw entries of the perturbation $\Delta$ independently from a Gaussian distribution with zero mean and  $v_\Delta$ variance. The measurements are assumed to be noiseless, i.e., $v_n = 0$ and $\alpha=1$. In \Cref{fig:uncertainties}, we plot the NMSE performance as a function of the uncertainties, i.e., $v_\Delta$, when $m=1500$. We can see that our algorithm has the best NMSE performance which shows that our algorithm is more robust to the additive measurement matrix uncertainties. 
\begin{figure*}
\subfloat[MNIST Reconstruction]{
\includegraphics[width = \columnwidth]{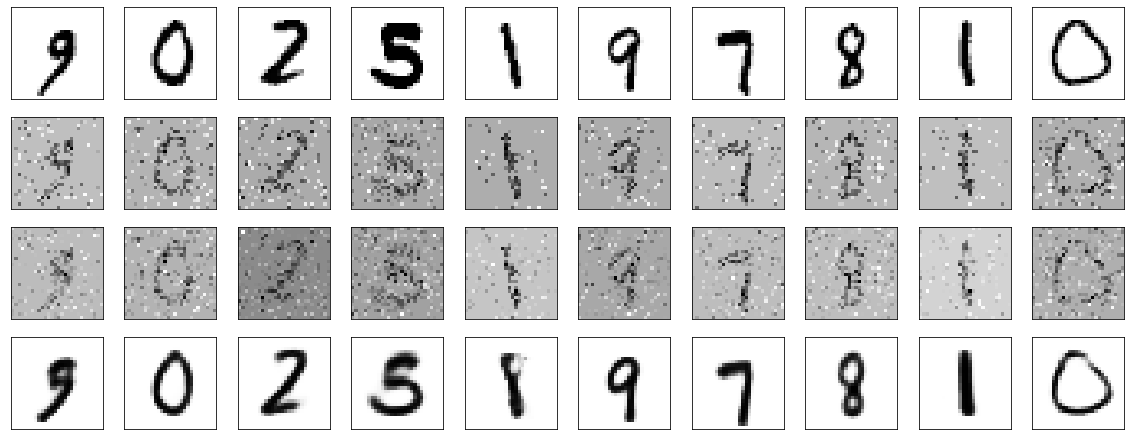}
}
\subfloat[FMNIST performance]{\hspace{-0.3cm}
\includegraphics[width=\columnwidth]{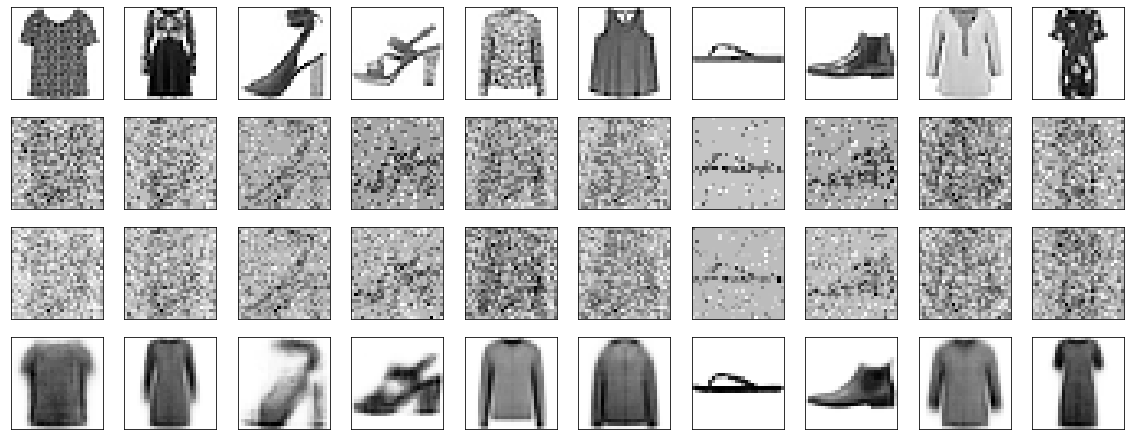}
}
\caption{The first row shows the original images, the second, third and fourth rows are the reconstruction images using \texttt{BIHT}, \texttt{YP} and our algorithm, respectively when $m = 784$ in the noiseless setting.}
		\label{fig:reconstructionImages}
\end{figure*}

\begin{figure}
\centering
\includegraphics[width= \linewidth/2]{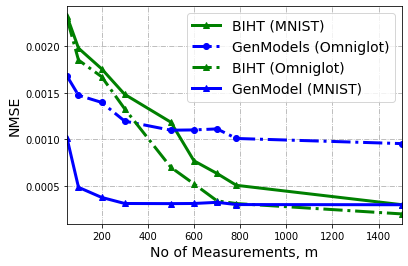}
\caption{NMSE values of our algorithm and \texttt{BIHT} using the MNIST and Omniglot datasets as a function of the number of measurements.}
\label{fig:omniglot}
\end{figure}

\subsection{Limitations}
In this subsection, we study the limitations of the deep generative model-based algorithm. Specifically, we look into the performance of our algorithm when the range space of the generative model does not faithfully represent the distribution of the compressed signal. There could be two possible cases as follows:
\subsubsection{Choice of Generative Model Architecture} It is necessary that the generative model architecture be well suited to the data that we need to learn. To investigate this effect, we study the reconstruction performance of the MNIST and the Fashion-MNIST datasets with the same VAE architecture. In \cite{fashionMnist}, the data distribution of Fashion-MNIST is shown to be more complicated compared to MNIST through tSNE visualization \cite{maaten2008visualizing}. We train VAEs for these two datasets with the same neural network architecture, over the same number of epochs with the same learning rate, and with the same optimizer. In this setup, we consider noiseless one-bit measurements for signal reconstruction. The reconstructed images in  \Cref{fig:reconstructionImages} show the superior visual quality of images from the MNIST dataset compared to the Fashion-MNIST dataset. For instance, in the clothing images in the Fashion-MNIST dataset, while the generative model learned the basic shapes of clothes, it fails to capture the finer details, such as patterns and design. Therefore, for the same network architecture, VAEs learned the distribution of the MNIST dataset but struggled to capture the more complex data distribution of the Fashion-MNIST dataset. 
From this experiment, we can conclude that the reconstruction performance of our algorithm is sensitive to the data distribution and the choice of the neural network architecture. In this scenario, we can choose generative models with better architecture for learning signal distribution, such as deep convolutional generative adversarial network or convolutional VAE.

\subsubsection{Different Data Distributions for Training and Testing} 
In practical scenarios, the signal distribution can change over time. To emulate this setup, we evaluated our algorithm, \texttt{Genmodel}, on the data from the Omniglot data. It is important to note that the generative model was trained on the MNIST dataset. The resulting NMSE values are plotted in  \Cref{fig:omniglot}. As anticipated, NMSE values for the data samples from the Omniglot dataset are higher than those from the MNIST dataset due to the high representation error (the first term in \eqref{eq:l_bound} of \Cref{thm:NN_error}). The representation error measures the distance between the compressed signal and the closest signal in the range space of the generator. This representation error increases with an increase in the mismatch between the training and the testing data distribution. For the same reason, the performance of the \texttt{BIHT} algorithm is better than \texttt{Genmodel} when $m$ is greater than $300$, and the difference in the NMSE values between the \texttt{BIHT} algorithm and  \texttt{Genmodel}  increases with $m$. However, for the MNIST dataset, the \texttt{BIHT} algorithm matches the performance of \texttt{GenModel} when $m = 1500.$ Therefore, it is crucial that the trained generative model accurately represents the signal distribution of the compressed signal. In addition, any shift in the signal distribution over time can result in degradation of reconstruction performance. Such shifts should be identified, and the generative models should be retrained to mitigate the issue.

\section{Conclusion}
We presented a one-bit CS algorithm using generative models. Unlike prior works on this topic, our approach learns the underlying structure of the signal without explicitly depending on any sparsity model. We also established reconstruction guarantees for the algorithm by characterizing the number of measurements that can achieve a given estimation error. Further, we empirically showed that our algorithm requires significantly less number of measurements for good reconstruction performance.
In contrast to the traditional algorithms, our algorithm recovered signals with both amplitude and direction information. 
Note that all the theoretical results presented in this work are applicable only to Gaussian measurement matrices. Extending these results to a broader class of measurement matrices is a direction for future work.

\appendices
\crefalias{section}{appendix}

\section{Toolbox}
We start with the mathematical toolbox, definitions, and results from the literature required to prove the results in the paper.

\begin{lemma}[{\cite[Lemma 4.1]{plan2012robust}}]\label{lem:lambda}
Let $\vecx^*,\vecx\in\bbR^N$ be such that $\lV\vecx^*\rV= 1$. Also, let the function $f_{\vecx^*}(\vecx)$ be
\begin{equation}\label{eq:f_fnt}
 f_{\vecx^*}(\vecx)   = \frac{1}{m}\sum_{i=1}^m\vecy_i\matA_i\tran\vecx,
\end{equation}
where $\vecy_i=\psi(\matA_i\tran\vecx^*)$ for some random function $\psi$ and $\matA_i$ is the $i$th row of $\matA$. 
Then,  if $\matA_{ij}\sim\Gauss(0,1/m)$, we have
\begin{equation}\label{eq:expect}
f_{\vecx^*}(\vecx) = \expect{}{\psi(a)a}\vecx\tran\vecx^*,
\end{equation}
where $a\sim\calN(0,1)$ follows the standard Gaussian distribution. Also, when $\psi(a)=\sign(a)$, we have $\expect{}{\psi(a)a} = 
        \sqrt{\frac{2}{\pi}}$.
\end{lemma}

We next need the notion of Gaussian mean width of a given bounded set, defined as follows.
\begin{defn}[Gaussian mean width~{\cite[Section 1.3]{plan2012robust}}]\label{eq:def:GaussianWidth}
The Gaussian mean width of a set $\calS$ is given by
\begin{equation}
\wid(\calS) = \expect{\vecg\sim\Gauss(0,\eye)}{\underset{\vecx,\vecx'\in\calS}{\sup}(\vecx-\vecx')\tran\vecg}.
\end{equation}
\end{defn}
Using the Gaussian mean width, we have the following concentration inequality. 
\begin{lemma}[{\cite[Proposition 4.2]{plan2012robust}}]\label{lem:concen}
Let $\vecx^*,\vecx\in\bbR^N$ and the function $f_{\vecx}(\vecx)$ be as defined in \Cref{lem:concen}. 
Then,  if $\matA_{ij}\sim\Gauss(0,1/m)$, and for any $t>0$ and set $\calT$, we have
\begin{multline}\label{eq:concen}
    \bbP\lc\underset{\vecx,\vecx'\in\calT}{\sup}\lv f_{\vecx^*}(\vecx-\vecx') \!-\! \expect{}{f_{\vecx^*}(\vecx-\vecx')}\rv \!\geq\! \frac{4\wid(\calT)}{\sqrt{m}}\!+\!t\!\rc\\
    \leq 4\exp\lb-\frac{mt^2}{8}\rb.
\end{multline}
\end{lemma}
The next lemma bounds the Gaussian mean width of a finite set.
\begin{lemma}[{\cite[Section 2.1]{plan2012robust}}]\label{lem:GWidth}
If $\calT$ is a finite set, then there exists a constant $C'>0$ such that 
\begin{equation}
    \wid(\calT)\leq C'\sqrt{\log \lv\calT\rv}.
\end{equation}
\end{lemma}
We can extend the above result to bounded sets using set cover and the following lemma help to bound the cardinality of a cover set.
\begin{lemma}[{\cite[Section 4.2.1]{vershynin2018high}}]\label{lem:covering}
For any $t>0$, there exist a $\calT$ such that it is the $t-$cover of a Euclidean ball in $\bbR^s$ with radius $r$, and its cardinality satisfies 
\begin{equation}
\lv\calT\rv\leq \lb\frac{4r}{t}\rb^{s}.
\end{equation}
\end{lemma}
Finally, the following result characterizes the Lipschitz property of a neural network with Lipschitz activation functions. 
\begin{lemma}[{\cite[Lemma 8.5.]{bora2017compressed}}]\label{lem:lipschitz}
Let $G$ be a $d-$layer neural network with at most $N$ nodes per layer, all weights are upper bounded by  $ w_{\max}$ in absolute value, and the non-linearity after each layer is $L-$Lipschitz. Then, the function $G$ is $(L Nw_{\max})^d-$Lipschitz.
\end{lemma}

\section{Proof of Theorem 1}\label{app:NN_error}
At a high level, the main steps of the proof are as follows: 
\begin{enumerate}[label={[\Alph*]},leftmargin=1cm]
\item We first prove that, for any $\beta>0$, the following statement holds with probability at least $1-4\exp\lb-2\beta^2\rb$,
\begin{equation}\label{eq:concentration}
\lV\hat{\vecx}-\vecx^*\rV^2\leq \lV\bar{\vecx}-\vecx^*\rV^2+\delta+4\sqrt{\frac{2\pi}{m}}\lb\wid(\calS)+\beta\rb,
\end{equation}
where $\wid(\calS)$ is the Gaussian mean width of the range $\calS$, and $\bar{\vecx}=\underset{\vecx\in\calS}{\arg\min}\lV\vecx^*-\vecx\rV$.
\item Next, we show that there exists a constant $C'>0$ such that $\wid(\calS)$ satisfies the following for any $r>0$:
\begin{equation}
\wid\lb \calS\rb\leq 8r\sqrt{s}+C'\sqrt{sd\log( LNw_{\max})}.
\end{equation}
\item Finally, we combine the above steps to bound the error $\lV\hat{\vecx}-\vecx^*\rV^2$ using an appropriate choice of $\beta$.
\end{enumerate} 
The details of each of the above steps are presented below.

\subsection{Gaussian Mean Width-based Probabilistic Bound}
Let $\vecx$ be such that 
\begin{equation}
    \vecx = \underset{\vecx'\in\calS,\lV\vecx'\rV\leq 1}{\arg\min}\; l_{\mathrm{loss}}(\vecx').
\end{equation}
By assumption, $\hat{\vecx}=G(\tilde{\vecz})$ minimizes the cost function in \eqref{eq:cost} over $\calS$ to within additive $\delta$ of the optimum. Thus, we get that
\begin{equation}
l_{\mathrm{loss}}(\hat{\vecx}) \leq l_{\mathrm{loss}}(\vecx)+\delta  \leq l_{\mathrm{loss}}(\bar{\vecx})+\delta,
\end{equation}
where we define
\begin{equation}
    \bar{\vecx} = \underset{\vecx'\in\calS,\lV\vecx'\rV\leq 1}{\min}\lV \vecx'-\vecx^*\rV^2.
\end{equation}
Substituting for $l_{\mathrm{loss}}$ from \eqref{eq:cost}, we obtain
\begin{equation}
     \lV\hat{\vecx}\rV^2-\frac{\sqrt{2\pi}}{m} \vecy\tran\matA\hat{\vecx}\leq \lV\bar{\vecx}\rV^2-\frac{\sqrt{2\pi}}{m} \vecy\tran\matA\bar{\vecx} +\delta.
\end{equation}
Rearranging the terms, we obtain
\begin{equation}\label{eq:err_13}
\lV\hat{\vecx}\rV^2\leq \lV\bar{\vecx}\rV^2+ \sqrt{2\pi}f_{\vecx^*}(\hat{\vecx}-\bar{\vecx})+\delta,
\end{equation}
where $f_{\vecx^*}(\cdot)$ is defined in \Cref{lem:lambda} with  $\psi(a)=\sign(a)$. 
Further, we use \Cref{lem:concen} with parameter $t=\frac{4\beta}{\sqrt{m}}$ to obtain that with probability at least $1-4\exp\lb-2\beta^2\rb$,
\begin{align}
f_{\vecx^*}(\hat{\vecx}-\bar{\vecx})& \leq \expect{\!}{f_{\vecx^*}(\hat{\vecx}-\bar{\vecx})}+\frac{4}{\sqrt{m}}\lb\wid(\calS)+\beta\rb\\
& =\expect{\!}{f_{\vecx^*}(\hat{\vecx})-f_{\vecx^*}(\bar{\vecx})}+\frac{4}{\sqrt{m}}\lb\wid(\calS)+\beta\rb\label{eq:err_11}\\
&= \sqrt{\frac{2}{\pi}}\lb\hat{\vecx}-\bar{\vecx}\rb\tran\vecx^*+\frac{4}{\sqrt{m}}\lb\wid(\calS)+\beta\rb,\label{eq:err_12}
\end{align}
where \eqref{eq:err_11} uses the fact the function $f_{\vecx^*}$ is linear and  we use \Cref{lem:lambda} with $\psi(\cdot)=\sign(\cdot)$ to get \eqref{eq:err_12}.
Substituting the above relation back into \eqref{eq:err_13} leads to
\begin{equation}
\lV\hat{\vecx}\rV^2\leq \lV\bar{\vecx}\rV^2+2\lb\hat{\vecx}-\bar{\vecx}\rb\tran\vecx^*+4\sqrt{\frac{2\pi}{m}}\lb\wid(\calS)+\beta\rb+\delta.
\end{equation}
Further, we note that
\begin{equation}\label{eq:indentity}
    \lV\bar{\vecx}\rV^2+2\lb\hat{\vecx}-\bar{\vecx}\rb\tran\vecx^*= \lV\bar{\vecx}-\vecx^*\rV^2+\lV\hat{\vecx}\rV^2-\lV\hat{\vecx}-\vecx^*\rV^2.
\end{equation}
Substituting the above relation into \eqref{eq:err_11} and rearranging the terms, we complete Step A.

\subsection{Bounding Gaussian Mean Width of Generator's Range}
The input to the generator $\vecz$ follows a uniform distribution, and therefore, we get that $\lV\vecz\rV\leq r$. This, in turn, implies that the range $\calS$ of $G$ satisfies 
\begin{equation}
\calS\subseteq G(\calB^s_r).\label{eq:subset}
\end{equation}

Next, using \Cref{lem:covering}, we construct a $\frac{t}{(LNw_{\max})^d}-$cover $\calT$ of $\calB^s_r$ such that its cardinality is upper bounded by 
\begin{equation}
    \lv\calT\rv\leq\lb\frac{4r(LNw_{\max})^d}{t}\rb^{s},
\end{equation}
for some $t>0$ which we choose later in the proof. Further, for any $\vecx'\in\calS$, there exists a point $\vecz$ such that $\vecx=G(\vecz)$, and for any point $\vecz\in \calB^s_r$, there exists a point $\bs{\tau}\in\calT$ such that
\begin{equation}\label{eq:z_Lipsc}
    \lV\vecz-\bs{\tau}\rV\leq \frac{t}{(LNw_{\max})^d}.
\end{equation}
Further,  we use \Cref{lem:lipschitz} to assert that generator function $G$ is $(L Nw_{\max})^d-$Lipschitz, leading to
\begin{equation}
    \lV\vecx-G(\bs{\tau})\rV=\lV G(\vecz)-G(\bs{\tau})\rV \leq (L Nw_{\max})^d\lV\vecz-\bs{\tau}\rV \leq t,
\end{equation}
using \eqref{eq:z_Lipsc}. Therefore, for any $\vecx'\in\calS$, there exists a point $T(\vecx')\in G\lb\calT\rb$
\begin{equation}
    T(\vecx')=\underset{\vect\in G\lb\calT\rb}{\arg\min}\lV\vect-\vecx'\rV,
\end{equation}
satisfying the bound $\lV \vecx'-T(\vecx')\rV\leq t$.
Thus, $G\lb\calT\rb$ is a  $t-$cover of $\calS$. Consequently,
\begin{equation}\label{eq:T_cardinality}
\lv G\lb\calT\rb\rv\leq \lv\calT\rv\leq \lb\frac{4r(L Nw_{\max})^d}{t}\rb^{s}
\end{equation}

Having constructed a finite cover $G\lb\calT\rb$,  we next bound the Gaussian mean width of $\calS$. For any vector, $\vecg\sim\Gauss(\zero,\eye)$,
\begin{align}
 \wid\lb \calS\rb &= \expect{}{\underset{\vecx_1,\vecx_2\in\calS}{\sup}\vecg\tran(\vecx_1-\vecx_2)}\\
&\leq \expect{}{\!\underset{\vecx_1,\vecx_2\in\calS}{\sup}\vecg\tran\lb\vecx_1-T\lb\vecx_1\rb\!+\!T\lb\vecx_2\rb-\vecx_2\rb}\notag\\
& \hspace{0.5cm}+\expect{}{\underset{\vecx_1,\vecx_2\in\calS}{\sup}\vecg\tran\lb T\lb\vecx_1\rb-T\lb\vecx_2\rb\rb}\label{eq:Gwidth_2}.
\end{align}
We further simplify the first term of the inequality using the Cauchy-Schwarz inequality as follows:
\begin{align}
\expect{}{\underset{\vecx_1,\vecx_2\in\calS}{\sup}\vecg\tran\lb\vecx_1-T\lb\vecx_1\rb\!+\!T\lb\vecx_2\rb-\vecx_2\rb}\notag\\
&\hspace{-6.5cm}\leq \expect{}{\lV\vecg\rV}\underset{\vecx_1,\vecx_2\in\calS}{\sup}\lV \vecx_1-T\lb\vecx_1\rb+T\lb\vecx_2\rb-\vecx_2\rV\label{eq:Gwidth_11}\\
&\hspace{-6.5cm}\leq \sqrt{\expect{}{\lV\vecg\rV^2}}\lb 2\;\underset{\vecx'\in\calS}{\sup}\lV \vecx'-T\lb\vecx'\rb\rV\rb\\
&\hspace{-6.5cm}\leq2t\sqrt{s}.\label{eq:Gwidth_12}
\end{align}
Similarly, simplifying the second term of \eqref{eq:Gwidth_2},
\begin{align}
\expect{}{\underset{\vecx_1,\vecx_2\in\calS}{\sup}\vecg\tran\lb T\lb\vecx_1\rb-T\lb\vecx_2\rb\rb}\notag\\
&\hspace{-4.5cm} \leq \expect{}{\underset{\vecx_1,\vecx_2\in G\lb\calB^s_r\rb}{\sup}\vecg\tran\lb \vecx_1-\vecx_2\rb}\leq \wid \lb G\lb\calT\rb\rb\label{eq:Gwidth_22}\\
&\hspace{-4.5cm}= C'\sqrt{2s\log \lb\frac{4r(LNw_{\max})^d}{t}\rb}.\label{eq:Gwidth_23}
\end{align}
Here, \eqref{eq:Gwidth_22} follows because $T(\vecx_1),T(\vecx_2)\in G\lb\calB^s_r\rb$, and thus, supremum in \eqref{eq:Gwidth_22} is over a larger set. Also, \eqref{eq:Gwidth_23} follows from \eqref{eq:T_cardinality} and \Cref{lem:GWidth} where $C'>0$ is the same as the constant in \Cref{lem:GWidth}. 

Further, combining \eqref{eq:Gwidth_2}, \eqref{eq:Gwidth_12}, \eqref{eq:Gwidth_23}, we get the following:
\begin{equation}
\wid\lb \calS\rb\leq 2t\sqrt{s}+C'\sqrt{2s\log \lb\frac{4r(LNw_{\max})^d}{t}\rb}.
\end{equation}
Finally, we choose $t=4r$ to complete Step B.

\subsection{Optimal $\beta$ Selection and Desired Bound}
Combining  Steps A and B, we get that with probability at least $1-4\exp(-2\beta^2)$
\begin{multline}\label{eq:theory_1}
\lV\hat{\vecx}-\vecx^*\rV^2\leq \lV\bar{\vecx}-\vecx^*\rV^2+\delta\\+4\sqrt{\frac{2\pi }{m}}\lb 8r\sqrt{s}+C'\sqrt{sd\log LNw_{\max}} +\beta\rb.
\end{multline}
As given in the statement of the theorem, let the following lower bound on $m$ holds for $C_1>64\pi$,
 \begin{align}\label{eq:num_mes}
m&\geq C_1\epsilon^{-2}s\lb 8r^2+C'd\log LNw_{\max} \rb\\
&\geq \frac{C_1}{2\epsilon^2}\lb 8r\sqrt{s}+C'\sqrt{sd\log LNw_{\max}} \rb^2
\end{align} 
If we choose $\beta= C_2\epsilon\sqrt{m}$ with $C_2=\frac{1}{4\sqrt{2\pi}}-\sqrt{\frac{2}{C_1}}>0$, 
\begin{equation}\label{eq:stepC_1}
\lV\hat{\vecx}-\vecx^*\rV^2\leq \lV\bar{\vecx}-\vecx^*\rV^2+\epsilon+\delta.
\end{equation}
 with probability at least $1-4\exp\lb-c\epsilon^2m\rb$. Finally, we also have
\begin{align}
\lV\bar{\vecx}-\vecx^*\rV^2 &= \underset{\substack{\vecz\in\bbR^s:\lV\vecz\rV_{\infty}\leq r\\\lV G(\vecz)\rV\leq 1}}{\min}\lV G(\vecz)-\vecx^*\rV^2\\
&\leq \underset{\substack{\vecz\in\bbR^s:\lV\vecz\rV\leq r\\\lV G(\vecz)\rV\leq 1}}{\min}\lV G(\vecz)-\vecx^*\rV^2.\label{eq:stepC_2}
\end{align} 
 Combining \eqref{eq:stepC_1} and \eqref{eq:stepC_2}, Step C is complete, and we arrive at the desired result. \hfill\qed
\section{Proof of \Cref{cor:noisy}}\label{app:noisy}
The proof is similar to that of \Cref{thm:NN_error} in \Cref{app:NN_error}, except that in \eqref{eq:err_13}, we define $f_{\vecx^*}(\cdot)$ as in \Cref{lem:lambda} with   $\psi(\cdot)=\eta_i\sign(\cdot)$. Here, we have
\begin{equation}
    \expect{}{\psi(a)a} = 
        (2\alpha-1)\sqrt{\frac{2}{\pi}},
\end{equation}
where $a\sim\calN(0,1)$, which changes \eqref{eq:err_12} as 
\begin{equation}\label{eq:f_bound}
f_{\vecx^*}(\hat{\vecx}-\bar{\vecx})
\leq \sqrt{\frac{2}{\pi}}(2\alpha-1)\lb\hat{\vecx}-\bar{\vecx}\rb\tran\vecx^*+\frac{4}{\sqrt{m}}\lb\wid(\calS)+\beta\rb.
\end{equation}
Substituting the above relation back into \eqref{eq:err_13} leads to
\begin{multline}
\lV\hat{\vecx}\rV^2\leq \lV\bar{\vecx}\rV^2+2(2\alpha-1)\lb\hat{\vecx}-\bar{\vecx}\rb\tran\vecx^*\\+4\sqrt{\frac{2\pi}{m}}\lb\wid(\calS)+\beta\rb+\delta.
\end{multline}
The above relation is equivalent to 
\begin{align}
    (2\alpha-1)\lb\lV\hat{\vecx}\rV^2- \lV\bar{\vecx}\rV^2-2\lb\hat{\vecx}-\bar{\vecx}\rb\tran\vecx^*\rb\notag\\
    &\hspace{-6cm}\leq (2\alpha-2)\lb\lV\hat{\vecx}\rV^2-\lV\bar{\vecx}\rV^2\rb+4\sqrt{\frac{2\pi}{m}}\lb\wid(\calS)+\beta\rb+\delta\\
    &\hspace{-6cm} \leq 2(1-\alpha)+4\sqrt{\frac{2\pi}{m}}\lb\wid(\calS)+\beta\rb+\delta,
\end{align}
where we use the fact that $0\leq \lV\hat{\vecx}\rV^2\leq 1$ and $0\leq \lV\bar{\vecx}\rV^2\leq 1$.
Simplifying the above relation using \eqref{eq:indentity} gives
\begin{multline}\label{eq:concentration_1}
\lV\hat{\vecx}-\vecx^*\rV^2\leq \lV\bar{\vecx}-\vecx^*\rV^2+\frac{2(1-\alpha)}{2\alpha-1}+\frac{\delta}{2\alpha-1}\\+\frac{4}{2\alpha-1}\sqrt{\frac{2\pi}{m}}\lb\wid(\calS)+\beta\rb.
\end{multline}
Following the rest of the proof as in \Cref{app:NN_error}, we derive the desired result. \hfill\qed
\section{Proof of \Cref{cor:noisy_moremes}}\label{app:noisy_moremes}
The proof is similar to that of \Cref{cor:noisy} in \Cref{app:noisy}, but we have to substitute \eqref{eq:f_bound} into the following relation instead of \eqref{eq:err_13},
\begin{equation}
\lV\hat{\vecx}\rV^2\leq \lV\bar{\vecx}\rV^2+ \frac{\sqrt{2\pi}}{2\alpha-1}f_{\vecx^*}(\hat{\vecx}-\bar{\vecx})+\bar{\delta},
\end{equation}
So, we deduce the following relation,
\begin{multline}
\lV\hat{\vecx}\rV^2\leq \lV\bar{\vecx}\rV^2+2\lb\hat{\vecx}-\bar{\vecx}\rb\tran\vecx^*\\+\frac{4}{2\alpha-1}\sqrt{\frac{2\pi}{m}}\lb\wid(\calS)+\beta\rb+\bar{\delta}.
\end{multline}
Using \eqref{eq:indentity}, we arrive at
\begin{multline}
\lV\hat{\vecx}-\vecx^*\rV^2\leq \lV\bar{\vecx}-\vecx^*\rV^2+\bar{\delta}\\+\frac{4}{2\alpha-1}\sqrt{\frac{2\pi}{m}}\lb\wid(\calS)+\beta\rb.
\end{multline}
Following the rest of the proof as in \Cref{app:NN_error}, we derive the desired result.
\hfill\qed
\bibliographystyle{IEEEtran}
\bibliography{Supporting_Files/IEEEabrv,Supporting_Files/bibJournalList,Supporting_Files/1-bitCS_cite}
\par\leavevmode
\end{document}